\documentclass[sigplan,10pt,nonacm]{acmart}

\renewcommand\footnotetextcopyrightpermission[1]{}
\usepackage{booktabs}
\usepackage{array}
\usepackage{amsmath}
\usepackage{amsfonts}  %
\usepackage{microtype}

\usepackage{amsthm}

\newtheorem*{definition*}{Definition}
\newtheorem{claim}{Claim}

\usepackage{enumitem}
\setlist{leftmargin=1.4em,topsep=2pt,itemsep=1pt,parsep=0pt}

\newenvironment{myitemize2}%
{\begin{list}{\labelitemi}{\itemsep1pt \topsep2pt \parsep0.00in
          \partopsep=0pt \leftmargin1.2em}}%
          {\end{list}}

\newcommand{\heading}[1]{\vspace{1ex}\noindent\textbf{#1}}
 
\begin{document}

\title{Defining AI-Native Systems:\\ Autonomy as Revision Authority}

\author{Cheng Tan}
\affiliation{%
  \institution{Northeastern University}
  \city{}
  \state{}
  \country{}
}
\email{}

\renewcommand{\shortauthors}{Cheng Tan}

\begin{abstract}
AI has begun to write systems code: agents now synthesize, verify, and deploy
system components.
Despite this shift, ``AI-native'' remains a marketing term with no precise
technical definition.
This paper gives it one.
We define AI-nativeness along a single axis---\emph{authority over the
system's own decisions}---rather than by the capability of the underlying AI
models.
Building on a decision-level model of a system, we distinguish
\emph{occupancy} (who executes a decision) from
\emph{revision authority} (who may change it),
organize revision authority into a ladder---self-tuning, self-rewriting,
self-architecting---and
define a system as AI-native when an AI autonomously rewrites the system's own
implementations.
The definition further requires an escalation detector, a verification
procedure, and a verified fallback, while leaving purpose and correctness
human-owned.
\end{abstract}
 

\maketitle

\section{Introduction}

Artificial intelligence is starting to write systems software. Beyond the
now-familiar practice of replacing a heuristic with a learned model, agentic
systems have begun to \emph{author} systems code: recent frameworks discover
novel scheduling and load-balancing algorithms that beat expert baselines, and
coding agents synthesize, benchmark, and verify implementations against real
workloads~\cite{adrs,alphaevolve,dwivedula2025vulcan}.
The question a systems designer faces has shifted accordingly---from
``can a model beat this heuristic on a benchmark?''
to ``how much of the system do we let an AI hold and change, and how do we stay
safe while it does?''

The moment is genuinely new. Machine learning for systems already
ships---learned indexes, caches, allocators, and schedulers run in
production---yet only as narrow, hand-built, hand-maintained point solutions,
because a model that \emph{resides} on the hot path pays for it in inference
overhead, workload drift, weak tail guarantees, and expensive bookkeeping
(\S\ref{sec:landscape}). What changed is the mode of use: an agent that
\emph{authors} code works off the critical path, at development time, and pays
none of those costs. For the first time it is feasible to let an AI continuously
revise a system's own \emph{implementation}---not merely retune its knobs---under
human-set goals.

That capability needs a name for the system property it creates, and
``AI-native'' is the label the field already reaches for. But the label is a
marketing term---attached to databases, operating systems, clouds, and developer
tools, meaning something different each time---and what it lacks is a rigorous,
technical statement of what property a system must have to earn it. This paper
offers one.

Our definition rests on a single axis: \emph{authority over the system's own
decisions}---who is permitted to change what the system does, at which level of
its design, how often, and under what guardrails. This is the property that
determines whether a system can \emph{evolve itself}, safely, and we argue it is
what ``AI-native'' should name. It is a claim about the \emph{system}, not the
AI: two recent efforts grade instead how much intelligence a problem
demands~\cite{maas} or how capable an AI is at systems research~\cite{feng},
whereas we grade how much authority a system delegates to an AI and how that
authority is bounded. The questions are orthogonal, and \S\ref{sec:def} makes the
relationship explicit.

We develop the definition in three steps.
\begin{myitemize2}

\item First, a \emph{decision-level model}
(\S\ref{sec:levels}) stratifies a system's decisions by binding
time---design, implementation, policy, runtime---and separates \emph{occupancy}
(who executes a decision) from \emph{revision authority} (who may change it), the
axis on which autonomy actually lives.

\item Second, a \emph{revision-authority ladder}
(\S\ref{sec:shift}) grades systems by how high that authority
reaches: self-tuning, self-rewriting, and self-architecting.

\item Third, the ladder framework yields the \emph{definition} (\S\ref{sec:def}): a system is AI-native iff an
AI autonomously rewrites its implementations under an escalation detector, holds
authority over how deciders are allocated, runs every revision through a
verification procedure while retaining a fallback, and leaves purpose and
correctness human-owned---a grade that must be \emph{evidenced} by an auditable
certificate, not asserted.
\end{myitemize2}

\noindent
We then draw an analogy between our definition and a precedent:
the SAE driving-automation levels (\S\ref{sec:analogy}), which pin
purpose and own failure in the same places we do, and close with the \emph{open
problems} (\S\ref{sec:disc}) that separate the definition from its practical
realization.

\section{AI for Systems Today, and Why Now}
\label{sec:landscape}

Machine learning for systems (ML4Sys) works, and it ships. Learned index
structures~\cite{learnedindex},
learning-based memory allocation in production server workloads~\cite{learnedmalloc},
admission optimization for datacenter flash caches~\cite{cachesack},
learned eviction for a large content-delivery network~\cite{halp},
learned virtual-machine {NUMA} placement~\cite{vmnuma}, and
lifetime-prediction-driven VM scheduling have all reached real deployments.

These are genuine successes. But they are also
\emph{narrow}: each is an expensive, component-scoped effort, hand-built and
hand-maintained by a team that understood both the model and the subsystem. ML
for systems is a collection of point solutions, not a pervasive, self-sustaining
property of how systems are built and evolved. That gap---between real
component wins and systemic adoption---is what motivates this paper.

\subsection{Why classic ML4Sys stayed narrow}
\label{sec:whynarrow}

The reason is not that the models were bad.
It is that a \emph{resident} learned component---for example, a neural network that executes on the live
path~\cite{nn4sysbench}---carries operational costs a small heuristic does not, and those costs
keep every deployment expensive and supervised.
For classic ML4Sys deployment,
five challenges remain, and each is a reason an
ML4Sys model stays expensive to own.

\heading{Runtime overhead.} Systems components run under punishing latency
budgets: a cache replacement resolves in nanoseconds, an allocator or scheduler
fires easily thousands of times per second. Neural inference---even a small
MLP---frequently costs more than the operation it optimizes, and the cost is not
only the forward pass but feature collection, preprocessing, and acting on the
prediction. Charged on the critical path at decision frequency, a learned
component is often a net \emph{slowdown} even when its decisions are
qualitatively better.

\heading{Distribution shift.} Workloads are nonstationary---diurnal cycles, load
spikes, new query shapes, data growth, tenant churn. A model trained on one
regime degrades off it, and learned indexes are the canonical case: excellent on
static read-only data, they force expensive retraining or lose performance under
insertions. Research benchmarks on fixed traces hide this and inflate apparent
gains. In production a mediocre-but-robust heuristic often beats a policy that is
excellent on its training distribution and unpredictable off it, and the cost of
keeping the model aligned---detecting drift, retraining, revalidating,
redeploying---is a continuous tax heuristics never incur.

\heading{Tail behavior.} Engineers care about p99/p999 and worst-case bounds,
not averages, and learned approaches are doubly mismatched. Tail events are rare,
so they are underrepresented in any collected dataset: the model has little
signal about exactly the regime that matters. And the objective works against
the tail: minimizing an aggregate loss is dominated by the common case, so the
optimizer will trade away tail behavior for a small average gain. Classical
heuristics come with understood worst-case bounds; learned components usually
offer good average-case behavior with no guarantee where a single bad decision
can cause a latency spike or an SLA violation. Recovering guarantees for a
resident network is possible but demands its own apparatus---verifying the
network against systems specifications, and generating those specifications in
the first place~\cite{ouroboros,nn4sysbench}---which adds to the operational costs.

\heading{Maintainability.} A 50-line heuristic becomes an ML pipeline: training
infrastructure, feature stores, model versioning, drift monitoring, retraining
triggers, and a serving path---each a new dependency and failure mode. Debugging
inverts too: a misbehaving heuristic can be read and patched; a misbehaving
network offers no comparably direct account of why it evicted a line. Systems
software has long lifespans and conservative deployment cultures, so operators
are reasonably reluctant to adopt an opaque component whose bookkeeping never ends and
whose marginal gains rarely justify its lifetime maintenance costs.

\heading{Expertise.} Applying ML to a component requires deep skill in two
domains at once---expertise in machine learning and systems---and the combination is rare. It bites
hardest in the data, labels, and reward: knowing which workloads matter, what the
``right'' decision is (often a systems question, sometimes unavailable because
only the counterfactual would reveal it), and how to encode value including the
tail. Get any wrong and the model looks good on paper and fails, or is dangerous,
in production.

None of this says ML for systems \emph{failed}. It says the resident-model
form---occupancy of the hot path---does not generalize into a way of
\emph{building} systems. Every barrier above is a cost of making a model
\emph{execute} on the live path; none is intrinsic to letting a model
\emph{shape} the system off it.

\subsection{What changed today}
\label{sec:whatchanged}

The mode of use changed. Agentic systems now \emph{author} systems code rather
than reside inside it: they discover new algorithms, evolve implementations
directly, and vet candidates through benchmarking and
verification~\cite{adrs,alphaevolve}. The AI moves from \emph{occupying} a
decision---a model on the hot path, deciding each event---to \emph{authoring the
code} that decides, in a control plane that runs far off the critical path. Maas
et al.\ name the resulting shape: a slow control plane where models reason about
long-tail policy and generate code, and a fast data plane where the generated
code executes with minimal latency~\cite{maas}. Early explorations already ask
whether an LLM can stand in for hand-tuned system policies
outright~\cite{zhao2025policies}. This dissolves the five barriers
at a stroke, because they were all costs of \emph{residency}: authored code pays
no inference overhead on the hot path, is re-authorable when the workload shifts,
can be verified for tail behavior before it ships, is as maintainable as the code
a human would have written, and shifts the expertise burden onto the agent.

This is why the moment matters. For the first time it is feasible to let an AI
continuously revise a system's own \emph{implementation}---not just retune its
knobs---under human-set goals. That capability is as consequential as it is
hazardous, and the field has no rigorous account of the system property it
creates, nor of the guardrails that must bound it.
Providing---or at least sketching---such an account is the task of the rest of this paper.
\section{A Decision-Level View of a System}
\label{sec:levels}

To say precisely what authority an AI holds, we need a vocabulary for the
decisions a system makes. This section builds one. We model a system as a
collection of decision points related by a dominance order, stratify that order
into levels by binding time, and then argue
that the resulting structure is \emph{all there is}: every decision that
affects system behavior lands at one level or on one of two orthogonal axes. The
definition of \S\ref{sec:def} then has somewhere to stand.

\subsection{Decision points, deciders, dominance}
\label{sec:prelim}

\heading{Workloads.} Let $W$ be the space of workloads. A workload $w \in W$ is
a stochastic process generating the inputs, requests, and environmental
conditions the system encounters. Workloads are nonstationary: the active
workload at time $t$ is $w_t$, and the gap between $w_t$ and whatever the system
was last tuned for is the pressure every adaptation mechanism exists to relieve.

\heading{Decision points.} A \emph{decision point} is a triple
$d = (X, I, J)$: an option space $X$ of admissible choices, the information $I$
available when the choice is made, and an evaluation criterion
$J : X \times W \to \mathbb{R}$ (latency, throughput, tail-bound satisfaction).

\heading{Decision and decider.}
A \emph{decision} at a decision point $d$ selects $x \in d.X$ given information $i \in d.I$.
A \emph{decider} for $d$ is any map $\delta : I \to X$; deciders range over
humans, static rules, heuristics, classic ML models, neural networks, and LLM agents.
Two questions must be kept apart:
\emph{who decides} (the decider $\delta$ occupying $d$)
and \emph{who may change the decider} (the authority to revise).
Conflating them is the error the rest of the paper is built to avoid.

\heading{Dominance.} Decision point $d$ \emph{dominates} $d'$, written
$d \succ d'$, if the outcome chosen at $d$ determines the option space,
information structure, or evaluation criterion of $d'$. Choosing at $d$ does not
\emph{answer} $d'$; it defines what $d'$ even is. Choosing a storage
architecture fixes which algorithms are admissible; an algorithm fixes which
parameters exist; a parameter setting fixes which runtime actions are possible.
The levels are the strata of this order.

\heading{Binding time.} Each decision has
a \emph{binding time}; that is, the point in
the system's life cycle at which its outcome becomes fixed.
Reading the dominance order by binding time is what makes the level classification comprehensive
(\S\ref{sec:closure}).

\subsection{The four levels}

Stratifying $\succ$ by binding time yields four levels (L1--L4) inside the artifact,
generated by one recursion spelled out in \S\ref{sec:genrule}.

\begin{myitemize2}
\item \textbf{L1: Design.} An L1 decision selects the decision architecture
  itself---the system's design: the set of top-level components,
  the interfaces between them,
  the objective (e.g., \emph{which} statistics matter---mean vs.\ p999),
  and the correctness invariants and worst-case constraints.
  Fixing the design $D$ thereby determines $P_D$, the set of implementations
  admissible at all.
  Examples: whether there is a buffer pool with an eviction decision at all;
  the storage/compute split; the SLA; which invariants may never be violated.

\item \textbf{L2: Implementation.} An L2 decision selects a program
  $p \in P_D$ behind a fixed interface, thereby fixing both the policy
  family $\Theta_p$ that L3 can express
  and the representation in which those L3 objects live.
  Examples: a hash index vs.\ a B-tree
  behind one lookup interface; which algorithm implements the scheduler;
  replacing a formula with a lookup table; whether the policy family is
  ``threshold rules'' or ``a small decision tree.''

\item \textbf{L3: Policy.} An L3 decision selects a policy
  $\theta \in \Theta_p$ from the family the current implementation exposes,
        binding the map $\pi_\theta : S \times I \to A$ (with $S$ the system state,
        $I$ the observed information, and $A$ the action space); it generates every
  subsequent L4 decision.
  Examples: LRU
  vs.\ LFU behind a flag; a cost model's constants; an admission threshold;
  timeout and retry parameters; and the trained weights of a runtime ML
  model.

\item \textbf{L4: Runtime.} An L4 decision is a single application of the current
  policy: given runtime state $s$ and observation $i$, select action
  $a \in A_s$, consumed the moment it is made.
  Examples: evict \emph{this} line; schedule \emph{this} task; take
  \emph{this} branch; choose \emph{this} join order.

\end{myitemize2}

\noindent
Above L1 sits \textbf{L0}, the \emph{purpose}---what the system is for, which
workloads matter, what ``correct'' means. L0 dominates L1 in the sense of
$\succ$ and, we will argue, must stay human-owned.

\heading{Runtime ML, classified.} A learned component in the classical
ML-for-systems sense decomposes exactly onto these levels: its weights are an L3
object selected by a training procedure (an L3 decider), and its forward pass is
an L4 decider whose cost is charged per triggering event.
The paradigm thus pays L4-frequency costs to exploit L3-frequency
learning---which is precisely the runtime-overhead barrier of
\S\ref{sec:landscape}.

\subsection{The decision-level hierarchy}
\label{sec:genrule}

The four levels are not an arbitrary list; they are generated by one recursion:

\begin{quote}
\emph{A level-$k$ decision selects an element of a space that the
level-$(k{-}1)$ decision produced.}
\end{quote}

$D$ (L1) defines $P_D$; $p \in P_D$ (L2) defines $\Theta_p$;
$\theta \in \Theta_p$ (L3) defines $\pi_\theta$;
$\pi_\theta$ (L4) defines the action. Equivalently, each level is
a stratum of $\succ$: L1 $\succ$ L2 $\succ$ L3 $\succ$ L4, and within a stratum
no decision dominates another.

\begin{table}[t]
\centering
\footnotesize
\setlength{\tabcolsep}{4pt}
\renewcommand{\arraystretch}{1.2}
    \begin{tabular}{@{}l@{\quad}c@{\quad}c@{\quad}c@{\quad}c@{}}
\toprule
 & \textbf{L1} & \textbf{L2} & \textbf{L3} & \textbf{L4} \\
 & \emph{design} & \emph{impl.} & \emph{policy} & \emph{runtime} \\
\midrule
Binding time         & design     & compile & config   & execution \\
\midrule
Decision frequency   & rare       & low    & medium   & extreme \\
\midrule
Per-decision stakes  & extreme    & high   & medium   & tiny \\
\midrule
Latency to decide    & months     & days   & ms--hrs  & ns--ms \\
\midrule
Revision cost        & interfaces & code   & params   & -- \\
\midrule
Decider today   & human      & human + agent  & tuner & code \\
\bottomrule
\end{tabular}
\caption{The four in-artifact levels with two opposed gradients: per-decision
stakes grow upward exactly as frequency grows downward. This opposition is the
formal content of the runtime-overhead barrier: expensive deciders are
affordable only where decision frequency is low, and L4's frequency excludes
expensive deciders.}
\label{tab:levels}
\end{table}

Table~\ref{tab:levels} lays out the binding time, frequency, and stakes of each
level, and it explains why one cannot simply ``put the smart thing on the hot
path'': the level whose decisions matter least per event is the one that fires a
billion times a second, and it admits only the cheapest decider. Intelligence
has to act where it is affordable---at L2 and L3---and let compiled code carry
L4.

\subsection{Adaptation, formally}
\label{sec:adaptation}

\begin{definition*}[Adaptation at level $k$]
An adaptation at level $k$ is a revision of the level-$k$ selection, triggered
by information observed after the original selection, holding all levels above
$k$ fixed:
\[
x_k \to x_k' \quad\text{with } x_{<k}\text{ fixed, triggered by }\Delta w.
\]
\end{definition*}

Under this definition, online tuning and retraining are L3 adaptation; an agent
rewriting the eviction routine against fresh traces is L2 adaptation;
re-architecting interfaces is L1 adaptation. L4 admits \emph{no} adaptation:
an L4 decision is consumed at its binding time, so only its generators (L3 and
above) can adapt. The \emph{adaptive range} of a system is the set of levels at
which it can adapt without human intervention. Classical systems have adaptive
range $\{$L3$\}$ at most, and typically only within a small hand-designed
$\Theta$ (i.e., a policy set at L3). Everything the next two sections grade is a claim about adaptive range
and about \emph{who} triggers movement within it.

\subsection{Are the four levels comprehensive?}
\label{sec:closure}

A stratification is only worth building a definition on if it is complete.
We claim the levels are
comprehensive in a precise, bounded sense, and we are equally precise about
where the boundary lies, because the two places the vertical hierarchy
\emph{cannot} see are exactly where the AI-native definition does its work.

\begin{claim}[Closure within artifact scope]
Fix a representation convention $R$ (what counts as ``code'' vs.\ ``data''). Then
every decision whose outcome affects the system's behavior on some workload
belongs to exactly one of L1--L4.
\end{claim}

\begin{proof}[Argument]
Classify by binding time and revision cost. Either the decision binds at
execution---L4---or it binds earlier. If earlier, either it is revisable without
changing code under $R$---L3---or it requires changing code. If it requires
changing code, either the change respects the current interfaces, objective, and
invariants---L2---or it does not---L1. The four cases are exhaustive and
mutually exclusive by construction.
\end{proof}

The claim is real, but its scope is exactly the phrase ``within artifact
scope.'' Three caveats bound it, and each names something the linear hierarchy
cannot represent, and each turns out to be essential for the definition of AI-native systems.

\heading{(a) The L2/L3 boundary is conventional, not intrinsic.} Whether a
quantity is a ``parameter'' or ``code'' is itself a design choice. A threshold in
a config file is L3; the same threshold as a compiled-in constant is L2; a
policy encoded as an interpreted rule table is L3 under the interpreter but L2 if compiled. %
Indeed, the level structure reflects
a continuous binding-time spectrum, and the partition depends on
$R$, which is decided at L1.

This is not a defect but a design choice: \emph{a
designer can move a decision between levels}---widening $\Theta_p$ pushes an L2
decision down to L3, where adaptation is cheaper. Much of what ``AI-native
architecture'' should mean is choosing $R$ so the decisions that drift fastest
live at the levels where revising them costs least.

\heading{(b) The hierarchy is missing its top: L0.} L1 selects system design,
but choosing \emph{what the system is for}---which workloads matter, what the SLA
should be, what ``correct'' means---dominates L1 in the sense of $\succ$ and is a
distinct decision. Call it L0. It is cleanly separable: two teams can share L0
(same purpose, same SLA) and diverge at L1 (different architectures). Folding L0
into L1 hides precisely the level every capability account leaves human.
Making L0 explicit is what lets a definition say \emph{where} autonomy stops.

\heading{(c) Two whole axes escape the vertical hierarchy.}
The levels classify decisions \emph{inside} the artifact.
Two families of decisions affect behavior from \emph{outside} it:
\begin{itemize}
\item \textbf{The allocation axis $\alpha$.} For each decision point $d$, some
  process selects the decider $\delta_d$---who or what decides here, with what
  fallback, revised when. These are decisions \emph{about the decision
  structure}, not about behavior, and they attach orthogonally to every level
  (one asks ``who decides?'' of an L3 tuning as of an L4 eviction). The
  allocation map $\alpha$, which assigns a decider to each decision point, has no
  home on a purely vertical axis.
\item \textbf{The verification axis $\rho$.} Whether a candidate revision is
  valid, when it deploys, how it rolls back, what is monitored, when drift
  triggers regeneration. These bind \emph{between} versions of the artifact
  rather than within one. Collapsing them into L2 loses the operational
  distinction that matters most: an L2 decision \emph{produces} a candidate; a
  $\rho$ decision \emph{admits} it.
\end{itemize}

\heading{Verdict.} The four levels are comprehensive as a stratification of
in-artifact decisions by binding time, relative to a fixed $R$.
However, they are \emph{not} a complete map of the decision
landscape an AI-native definition must range over. The complete picture is a
vertical hierarchy of five levels crossed with two orthogonal axes:
\[
\underbrace{\text{L0}}_{\text{purpose}} \;\succ\;
\underbrace{\text{L1} \succ \text{L2} \succ \text{L3} \succ \text{L4}}_{\text{artifact axis}}
\quad\times\quad
\underbrace{\alpha}_{\text{allocation}}
\quad\times\quad
\underbrace{\rho}_{\text{verification}}.
\]
The definition that follows (\S\ref{sec:def}) factors cleanly along exactly these
three directions---adaptive range on the vertical axis, authority over $\alpha$,
constraint by $\rho$, with the human residue pinned at L0.

\subsection{Occupancy versus revision authority}
\label{sec:occupancy}

The pivotal move, which the allocation axis makes available, is to separate two
questions the field routinely runs together.
\emph{Occupancy:} who executes the selection at a decision point---compiled code, a lookup table, or a neural
policy's forward pass?
\emph{Revision authority:} who is allowed to \emph{change} that decision when the workload demands it?
A learned eviction policy \emph{occupies} L4---its forward pass is the decider---while holding
revision authority over nothing: it cannot re-tune its own weights, rewrite its
serving code, or alter its design.
A human installs it and a human maintains it.
\emph{Occupancy without revision authority is not autonomy; it is an expensive resident.}
Autonomy is a claim about revision authority, and that is the axis the
next two sections climb.
Table~\ref{tab:worked-levels} presents several canonical examples that illustrate these differences.

\begin{table*}[t]
\centering
\small
\setlength{\tabcolsep}{6pt}
\renewcommand{\arraystretch}{1.15}
\begin{tabular}{@{}p{0.42\textwidth}cp{0.42\textwidth}@{}}
\toprule
\textbf{Decision} & \textbf{Level} & \textbf{Notes} \\
\midrule
Evict this cache line now & L4 & consumed per event \\
Cache policy $=$ LRU (runtime flag) & L3 & parameter under $R$ \\
Cache policy $=$ LRU (only policy compiled in) & L2 & needs a code change \\
Neural eviction model's weights & L3 & training is an L3 decider \\
\quad its forward pass & L4 & inference charged per event \\
B-tree $\to$ learned index, same API & L2 & interfaces preserved \\
Add an observable feature to eviction state & L1 & changes $S$, hence $I$ below \\
``p999 $<$ 2\,ms is the goal'' & L0 & dominates $J$ at L1 \\
``Agent tunes; humans approve diffs'' & $\alpha$ & attaches to L3/L2 \\
``Every policy must pass the trace suite'' & $\rho$ & admission, not behavior \\
\bottomrule
\end{tabular}
\caption{Worked classifications. The last two rows fall on the orthogonal axes:
allocation $\alpha$ and verification procedure $\rho$.}
\label{tab:worked-levels}
\end{table*}

\section{The Revision-Authority Ladder}
\label{sec:shift}

We now grade systems by \emph{revision authority}: the highest level at which an
AI process can perform adaptation autonomously.

\begin{definition*}[Autonomy ceiling]
The \emph{autonomy ceiling} of a system is the highest level Lk in the
dominance order $\succ$ at which an AI process can perform adaptation
autonomously---i.e., at which Lk lies in the adaptive range with an AI as the
adapting decider. Grade Sk names a system with ceiling Lk.
\end{definition*}

Two structural facts fix the shape of the ladder. First, because a level-$k$
selection re-parameterizes every level below it (the generating rule of
\S\ref{sec:genrule}), revision authority is \emph{cumulative downward}: authority
at Lk without authority at the level just below it is incoherent, since selecting
a new $p$ entails selecting an initial $\theta \in \Theta_p$, and selecting a new
$D$ entails selecting an initial $p \in P_D$. A grade-$k$ system therefore has
adaptive range covering L3 through Lk under AI authority. Second, L4 is
excluded from the ladder by the formalism itself: an L4 decision is consumed at
binding time and admits no adaptation (\S\ref{sec:adaptation}). ``Autonomy at
L4'' is a category error, so the ladder has exactly three grades, each indexed by
the level it governs.

\heading{Off-ladder: model-residency.} A system is \emph{Model-Resident} iff
some runtime decision point is occupied by a learned model:
$\alpha(d) = \pi_\theta$ with $\theta$ selected by training. This is an
\emph{occupancy} property, not a grade: the weights are an L3 object and the
forward pass an L4 decider, and a Model-Resident component may sit inside a
system of any grade or none. Crucially, it confers no autonomy by itself---a
learned index bolted into a database revises nothing; a human made the allocation
once and a human maintains it. A large body of ML for systems optimized occupancy
of the one level where adaptation is undefined, while the ladder below stood
unclimbed.

\heading{S3: Self-Tuning.}
An AI process $\tau$ revises the policy $\theta$ within
  the family the implementation exposes ($\theta \to \theta'$, with $p$ and $D$
  fixed): autoscalers with fitted forecasters, knob tuners, self-tuning database
  configurations. Runtime ML is the degenerate case---$\tau$ is a training loop,
  the L3 object a weight vector, the L4 occupant the resulting forward pass. Its
  constitutive failure is \emph{saturation}: when $\Delta w$ outruns what
  $\Theta_p$ can express, no $\theta'$ recovers, and the degradation is silent,
  because ``is $\Theta_p$ still the right family?'' is an L2 question, one level
  above the system's authority. An S3 system cannot even represent its own
  inadequacy. This is the classical ceiling: essentially every adaptive system
  built before LLM coding agents is S3 or below.

\heading{S2: Self-Rewriting.} An AI process $\sigma$ synthesizes a new
  \emph{implementation} $p' \in P_D$ behind the same interfaces ($p \to p'$,
  design fixed): swap the threshold rule for a decision tree, the hash index for
  a range index, re-derive the eviction routine against fresh traces. This is the
  level that was out of reach before, because revising an implementation means
  writing code against a specification---exactly what agentic synthesis now
  does~\cite{adrs,alphaevolve}.

  Two clauses are constitutive at S2 that were
  optional at S3. \textbf{(i)~Escalation.} An \emph{escalation detector}
  $\varepsilon$ maps evidence of saturated L3 adaptation to an L2 regeneration
  trigger, turning ``my tuning has stopped working'' into a signal rather than an
  incident; it is the genuinely novel object at this level---a decision procedure
  over the adequacy of decision procedures. \textbf{(ii)~$\rho$-verification.}
  Every candidate revision passes a verification procedure $\rho$---validation
  suites, invariant checks, trace replay, adversarial probes---with a verified
  non-AI fallback retained, so the deployed artifact is
  $\text{fallback} \oplus p'$ and its worst case is the fallback's by
  construction. S2 without $\rho$ is not a maturity grade; it is an incident
  generator.

\heading{S1: Self-Architecting.}
An AI process revises the \emph{design}
  itself, including the convention
  $R$ that fixes what counts as a parameter versus code, subject only to a
  human-owned L0. Three properties make S1 qualitatively, not incrementally,
  harder: revising design changes what $\rho$ \emph{is} (self-referential
  validation---who checks the checker after it is rewritten?); authority over $R$
  lets the system \emph{re-stratify}, moving fast-drifting decisions to cheaper
  levels; and L1 revision is interface renegotiation with external stakeholders,
  so the practical ceiling is \emph{S1 within a negotiated envelope}, not
  unbounded self-design. This tier is aspirational; \S\ref{sec:disc} returns to
  its obstacle.

\heading{Weak and strong forms.} Each grade admits a \emph{weak} form, written
weak-Sk, in which the AI has the revision \emph{capability} but the trigger is
external---a human or a schedule invokes regeneration---and a \emph{strong} form,
Sk, that adds the autonomous trigger (for S2, the escalation detector
$\varepsilon$; for S1, an analogous detector over design adequacy). The
distinction is where the value lives: an organization running an agent that
rewrites heuristics \emph{on request} has a weak-S2 system---synthesis cost
has collapsed, but drift detection remains a human task.
The step from weak-S2 to strong-S2 automates the \emph{trigger}, not the synthesis, and that is
where most of the drift-robustness value sits.

Crucially, climbing this ladder does \emph{not} move intelligence onto the hot
path. The correct occupant of L4 is compiled code at every grade---$\sigma$'s
output is a \emph{program}, and programs are what run at L4. The AI works at
development time and between versions, at L2 and L3 frequency, where its cost is
affordable. This is the same control-plane/data-plane split that Maas et al.\
argue for on capability grounds~\cite{maas}: heavy reasoning authors code off the
critical path; cheap generated code executes on it.
Our contribution is not that architecture but a \emph{governance} account of it:
who may revise what, and under which guarantees,
which we define in the next section.
\section{AI-Native, Defined}
\label{sec:def}

An AI-native system is not a system that contains AI models. It is a system whose
\emph{life cycle}---the cycle that revises the system as the world
changes---has been handed, under bounds, to an AI process. We define it by
three factors, one per axis of \S\ref{sec:levels}, all bound by a human-owned
envelope:

\begin{definition*}[AI-Native System]
A system is \emph{AI-native} if and only if it satisfies all of the following conditions:
\begin{enumerate}
    \item \textbf{Strong-S2 autonomy.}
    An AI autonomously revises system implementations and, via the escalation
    detector $\varepsilon$, autonomously determines when such revisions are needed.

    \item \textbf{Allocation authority.}
    An AI controls the allocation map $\alpha$ across levels L2--L4. At each
    decision point, it determines which decider occupies that point and may
    reassign the decider as cost and workload change.

    \item \textbf{$\rho$-verified revision.}
    Every AI-generated revision is subject to the verification procedure
    $\rho$, and the system retains a verified non-AI fallback implementation.

    \item \textbf{Human-owned envelope.}
    Level L0, together with the objective and invariants of L1, remains under
    human ownership and constrains all AI authority in the levels below.
\end{enumerate}
\end{definition*}

Compactly: \textbf{AI-native $=$ strong-S2 $\times$ $\alpha$-authority $\times$
$\rho$-verification, under a human-owned envelope.} Each factor answers a distinct
need. \emph{Strong-S2} supplies the revision authority a resident model lacks,
and its escalation detector $\varepsilon$ is the direct remedy for silent drift: the system
turns ``my policy has stopped working'' into a signal instead of an incident.
\emph{$\alpha$-authority} keeps compiled code on the data path and lets the AI
place an expensive decider only where its frequency can amortize it.
\emph{$\rho$-verification} makes autonomous rewriting survivable: because the deployed
artifact is $\text{fallback} \oplus p'$, the worst case of the composite is that
of the verified fallback, given a bounded, reversible switch with no
cross-version state corruption---the runtime-assurance pattern of the Simplex
architecture~\cite{simplex}.
And the \emph{human-owned envelope} marks where
autonomy stops: the AI may rewrite how the system pursues its purpose, never what
its purpose is.

The three factors are exactly the three axes of the extended decision landscape
of \S\ref{sec:closure}---adaptive range on the vertical hierarchy, the allocation
axis $\alpha$, and the verification axis $\rho$---with the human residue pinned at
L0. That the definition factors cleanly along these axes is evidence it carves
the space at its joints rather than stipulating a checklist.

\heading{Grade membership.}
\label{sec:certificates}
Grade membership should be demonstrated by an artifact, not claimed. The
\emph{certificate} of a grade is the object whose existence and auditability
establishes the authority (Table~\ref{tab:certs}). The certificate of an
AI-native system is the S2 row: an auditable, verified code diff $p \to p'$
carrying its passing $\rho$ record, its $\varepsilon$ trigger evidence, and its
retained fallback---an artifact an ordinary code review can check. Autonomy you
cannot audit is not a grade.

\begin{table}[t]
\centering
\small
\setlength{\tabcolsep}{4pt}
\renewcommand{\arraystretch}{1.25}
\begin{tabular}{@{}p{0.07\columnwidth}p{0.55\columnwidth}p{0.24\columnwidth}@{}}
\toprule
\textbf{Grade} & \textbf{Certificate} & \textbf{Auditable by} \\
\midrule
S3 & Tuning log: $(\Delta w,\ \theta \to \theta',\ \text{outcome})$ triples & monitoring review \\
\midrule
S2 & Code diff $p \to p'$ with passing $\rho$ record, $\varepsilon$ trigger, retained fallback & ordinary code review \\
\midrule
S1 & Re-derived $D'$ \emph{and} $\rho'$ with an argument that $\rho'$ preserves the invariants & design review (human) \\
\bottomrule
\end{tabular}
\caption{Certificates by grade. The S1 row shows why it
is hard: the certificate must include a re-derived admission process a human can
still audit, which caps the admissible complexity of any redesign at human
comprehension.}
\label{tab:certs}
\end{table}

\heading{The definition at work.}
Table~\ref{tab:worked-def} runs representative systems through the definition.
The pattern to read off it: occupancy varies freely (models, code, agents appear
at every row) while the grade tracks only revision authority, and AI-native is
the single row where strong-S2, $\alpha$-authority, and $\rho$ hold together.

\begin{table}[t]
\centering
\footnotesize
\setlength{\tabcolsep}{4pt}
\renewcommand{\arraystretch}{1.2}
\begin{tabular}{@{}p{0.60\columnwidth}p{0.30\columnwidth}@{}}
\toprule
\textbf{System} & \textbf{Grade} \\
\midrule
B-tree with hand-tuned fill factor & none \\
Learned index in production DB & none (Model-Resident) \\
\midrule
Autoscaler with fitted forecaster & S3 \\
Online-retrained learned cache & S3 (Model-Resident) \\
\midrule
Agent re-derives eviction code on request  & weak-S2 \\
\quad + drift monitor triggers it, fallback retained & S2 \\
\quad + AI also (re)assigns per-point deciders & \textbf{AI-native} \\
\bottomrule
\end{tabular}
\caption{The definition applied. Autonomy climbs by automating the
\emph{trigger} and widening \emph{authority}, not by making the resident model
smarter.}
\label{tab:worked-def}
\end{table}

\heading{Compared with recent related efforts.}
This definition grades a different object than the recent capability taxonomies,
and is orthogonal to both. Maas et al.\ ask how much \emph{intelligence} a
problem demands; Feng et al.\ ask how capable an \emph{AI} is at systems
research~\cite{maas,feng}. We ask how a \emph{system} governs its own evolution.
The axes are independent: a system built entirely on narrow, ``low-capability''
ML can be AI-native if an agent rewrites its heuristics under $\rho$-verification,
while a one-off superhuman result from an algorithm-discovery run is \emph{not}
AI-native, because it holds no standing authority and closes no loop.
The paradigm shift AI-native names is in who holds revision authority, not in how smart the
resident model is.
\section{A Precedent: Driving Automation}
\label{sec:analogy}

The definition's shape is not unique to systems.
The SAE J3016 driving-automation levels are the best-known precedent for
grading autonomy not by intelligence but by \emph{which decisions the machine holds, within what
bounds, and who owns failure}~\cite{sae}.
A decade of automotive autonomy paid,
in engineering effort, for lessons the systems community can now read off rather than re-learn.

\subsection{The mapping}

\begin{table}[t]
\centering
\small
\setlength{\tabcolsep}{4pt}
\renewcommand{\arraystretch}{1.2}
\begin{tabular}{@{}cp{0.52\columnwidth}p{0.20\columnwidth}@{}}
\toprule
\textbf{SAE} & \textbf{Systems analog} & \textbf{Grade} \\
\midrule
L0--L1 & Humans revise everything; at most a tiny hand-tuned $\Theta$ & none \\
L2 & Runtime ML bolted in: a learned component on the hot path, a human babysits drift and rollback & Model-Resident ($+$S3) \\
L3 & Agent-maintained heuristics: L3 re-tunes autonomously, structural drift triggers L2 regeneration or handoff & weak-S2 $\to$ S2 \\
L4 & Autonomous rewriting within a human-owned spec, verified fallback owning the worst case & S2 / AI-native \\
L5 & AI authority over the design $D$ itself & S1 (aspirational) \\
\bottomrule
\end{tabular}
\caption{The SAE driving-automation levels mapped onto the revision-authority
grades. The mapping is exact only after Lesson~4 relocates the comparison from
the vehicle to the fleet loop.}
\label{tab:sae}
\end{table}

Three correspondences beneath Table~\ref{tab:sae} are structural, not
decorative.

\heading{The ODD is the human-owned envelope.} An SAE L4 vehicle is \emph{fully}
autonomous only \emph{within} an operational design domain (ODD) that humans
define---geography, weather, road class. The systems counterpart is exact: L0
purpose plus the objective and invariants of L1 bound the region in which the
agent's revision authority is total and outside which it is zero. In both
domains, ``conditional'' is not a hedge on the definition; it is what makes the
definition operable.

\heading{The minimal-risk condition is the verified fallback.} What separates
SAE L3 from L4 is not driving skill but \emph{fallback ownership}: an L4 vehicle
must reach a safe state with no human available. That is precisely the
$\rho$-verification clause---deployed artifact $=$ fallback $\oplus$ optimized
policy, so the composite's worst case is the verified fallback's by
construction.

\heading{The passenger still picks the destination.} Even a hypothetical L5
vehicle does not choose where you are going. Destination selection dominates the
driving task in the sense of $\succ$---it defines what the drive \emph{is}---and
stays human at every level. This is L0, independently rediscovered: full
self-driving never meant the car chooses your errands, and AI-native must never
mean the system chooses its own SLA.

\subsection{Lessons borrowed}

\heading{Lesson 1: the dangerous level is the middle one.} The industry's
hardest-won finding is that SAE L2 is the \emph{worst} configuration: enough
autonomy to invite trust, not enough to own failure, with a human ``monitor''
whose vigilance decays exactly as the automation's competence grows. The
translation is uncomfortable and precise---runtime ML in production \emph{is} the
L2 configuration. A learned scheduler handles the common case well enough that the
human stops watching, degrades silently under drift, and returns control (via an
incident) at the moment of maximum difficulty.

\heading{Lesson 2: progress came from the envelope, not the model.} The jump
from demos to deployed L4 was achieved by ODD specification, validation
infrastructure, scenario coverage, and fallback engineering---not by better
driving networks alone. The capability was necessary; the \emph{envelope} made
it shippable. This is the thesis of our argument arriving from an independent
direction: the path to AI-native runs through $\rho$ and the specification, not
through smarter models on the hot path.

\heading{Lesson 3: the handoff problem is the escalation detector.} SAE L3's
notorious difficulty is the handback: the vehicle must know, with lead time, that
it is leaving its competence. This is exactly $\varepsilon$, converting ``my L3
adaptation has saturated'' into a trigger \emph{before} silent degradation
becomes an SLA violation. Self-assessment is the hard capability, harder than the
competence it assesses. In practice, L3's human handback proved so awkward that
L4---machine-owned fallback---became the more tractable target.
The analog: aim for verified-fallback composition rather than human-in-the-loop
handback under time pressure.

\heading{Lesson 4: autonomy lives in the loop, not the artifact.} The deployed
vehicle is, in our vocabulary, Model-Resident: it executes a frozen policy and
revises nothing. What made the leading programs work is the \emph{fleet loop
around} it---disengagement and telemetry (drift observation), offline
re-derivation (regeneration), simulation suites (validation), gated
over-the-air deployment (admission, i.e.\ $\rho$). That loop is an S2 system
whose artifact happens to be a driving stack. The lesson generalizes:
\textbf{autonomy is a property of the maintenance loop, not of the artifact}, and
asking ``is this system AI-native?'' of a binary is as confused as asking it of a
single parked car. The grade attaches to the system \emph{plus its loop}.

\section{Open Problems and a Call to Build}
\label{sec:disc}

Strong-S2 is close, and the obstacles are guardrails, not model quality. The
capability to synthesize an implementation already exists; what stands between it
and an AI-native system is the machinery that makes standing revision authority
safe. Four problems are load-bearing.

\begin{itemize}
\item \textbf{Escalation detection ($\varepsilon$).} Knowing when you don't know
  is the hard, under-studied capability. Detecting a saturated policy family
  before it becomes an SLA violation is the single highest-leverage piece of an
  AI-native system, and it is harder than the competence it guards: it is a
  decision procedure over the adequacy of decision procedures. The automotive
  handback problem (\S\ref{sec:analogy}) is the same problem in another field,
  and it was the one that gated deployment there too.
\item \textbf{Verification and fallback composition ($\rho$).} Trace-replay
  suites, adversarial workload probes, invariant checks, and provably-retained
  fallbacks are what make autonomous rewriting production-ready. Systems hold a
  real advantage here: they admit cheap, reliable verifiers and simulators
  against which a candidate can be scored without
  hallucination~\cite{maas}, which makes $\rho$ tractable in a way automotive
  simulation never fully was.
\item \textbf{Re-stratification.} Choosing the representation $R$ so the
  fastest-drifting decisions live at the lowest levels---where revision is
  cheapest and $\rho$ is lightest---may be the deepest sense of ``AI-native
  architecture.'' It is design work that pays off only once the system is
  expected to revise itself, which is why it has no real precedent to borrow.
\item \textbf{The threat surface.} An AI with standing authority to rewrite and
  deploy L2 code is itself an attack surface: a crafted workload can steer
  synthesis or trip $\varepsilon$ adversarially, and a compromised agent holds
  deploy authority. Guardrails here are a prerequisite, not a refinement.
\end{itemize}

The frontier beyond strong-S2 is S1, and its difficulty is self-referential:
revising the design changes what the levels below \emph{are}, including the
verification procedure $\rho$, whose suites are written against the design's
interfaces. Who checks the checker after the checker is rewritten? An S1
certificate must include a re-derived $\rho'$ a human can still audit, which caps
the admissible complexity of any redesign at human comprehension. So the
practical ceiling is S1 within a negotiated envelope, not unbounded self-design.

None of these is a reason to wait. Each is a concrete, buildable
artifact---a detector, a verifier, a representation, a threat model---and the
capability they would govern is already in the field. The definition names what
we are building toward; these are the guardrails that make building it safe.
\section{Related Work}
\label{sec:related}

\heading{Taxonomies of ML for systems.}
Prior efforts to organize the AI-for-systems space classify \emph{techniques
and their fit to problems}. Maas's taxonomy~\cite{maas2020taxonomy} provides
the canonical vocabulary for deciding whether and how ML applies to a systems
problem, and comprehensive surveys organize the resulting literature by domain
and learning paradigm~\cite{kanakis2022survey, wu2023survey}. In our
vocabulary, these works catalog \emph{occupancy}: which decision points a
learned model may reside at, and with which method.

\heading{Capability ladders for AI\@.}
A second lineage grades the \emph{AI} rather than the system. Morris et
al.~\cite{morris2024levels} import the spirit of the SAE driving-automation
levels~\cite{sae}.
Further, Maas et al.~\cite{maas} bring this ladder to systems, introducing a human-machine
baseline and arguing that fully solving classic system policies is
``AGI-complete.''
Mitchell et al.~\cite{mitchell2025fully} grade agent autonomy
by escalating authority over program flow, topping out at agents that rewrite
their own code---precisely what our S2 grade regulates, with escalation
detection and $\rho$-verification.
These ladders grade what an AI \emph{can do}; ours grades what a deployed system-plus-loop \emph{is permitted to
revise}, and evidences the claim with certificates.

\heading{Autonomy in specific system domains.}
Domain communities have independently built autonomy ladders for
\emph{operating} systems without granting them authority over their own
implementation. Autonomic computing introduced the MAPE-K loop and the
self-$\ast$ properties under human-supplied high-level
objectives~\cite{kephart2003vision}, a lineage the self-adaptive-systems
community has since systematized~\cite{weyns2020introduction}. Self-driving
databases~\cite{pavlo2017selfdriving, vanaken2017automatic} and self-driving
networks~\cite{feamster2018networks} pursue closed-loop operational autonomy,
while the 3GPP/TM~Forum autonomous-network levels~\cite{tgpp2023ts28100} grade
autonomy per task category---execution, analysis, decision, intent---with
intent held human longest. These frameworks reach at most our S3
(parameter-level adaptation within a fixed implementation); their per-category
grading and human-held intent prefigure, respectively, our decision-level
stratification and the human-owned L0.

\heading{AI-driven systems research.}
Most recently, coding agents have been applied to systems research itself.
ADRS~\cite{adrs, cheng2025letbarbarians} and
AlphaEvolve~\cite{alphaevolve} demonstrate LLM-driven discovery of
algorithms that outperform expert baselines, treating systems as white boxes
whose code the AI may rewrite offline; Feng et al.~\cite{feng} propose the SOAR
dimensions and five levels of ``system intelligence'' keyed to
PhD-student personas, and argue benchmarks are the binding constraint. These
efforts grade the \emph{research agent's} capability; the revision authority
they exercise lives in an offline loop with a human trigger---our weak-S2. Our
framework is complementary: it attaches the grade to the deployed artifact and
its maintenance loop, makes the escalation detector ($\varepsilon$) and the
verification procedure ($\rho$) constitutive, and pins purpose at a human-owned
L0---turning
their demonstrated capability into a checkable property of a system.
\section{Conclusion}

``AI-native'' has been a slogan in search of a referent; we have given it one.
AI is already writing systems code and moving into the control plane, and what
has been missing is a definition of the \emph{system property} this enables:
bounded authority---verified before each change takes effect---for an AI to
revise the system's own decisions under a human-owned purpose. We named it,
graded it, and located the guardrails that make its building safe. The
capability arrived ahead of the discipline; the definition is the first step.

\bibliographystyle{ACM-Reference-Format}
\bibliography{references}
 
\end{document}